\documentclass{article}
\usepackage{ijcai20}

\usepackage{times}
\usepackage{soul}
\usepackage{url}
\usepackage[hidelinks]{hyperref}
\usepackage[utf8]{inputenc}
\usepackage[small]{caption}
\usepackage{graphicx}
\usepackage{amsmath}
\usepackage{amsthm}
\usepackage{booktabs}
\usepackage{algorithm}
\usepackage{enumitem}
\usepackage{scalerel,stackengine}
\usepackage[noend]{algpseudocode}
\usepackage{natbib}
\usepackage{amsmath,amsfonts,amsthm,bm}

\usepackage{color}
\usepackage[normalem]{ulem}
\usepackage{booktabs}

\newcommand{\Bo}[1]{}
\newcommand{\comment}[1]{}
\usepackage{marginnote}

\newtheorem{theorem}{Theorem}
\newtheorem{lemma}[theorem]{Lemma}

\def\BCRPO{BRPO}
\def\acro{Acrobot-v1}
\def\cart{Cartpole-v1}
\def\luna{Lunarlander-v2}

\def\tb{\textbf}
\def\mdp{\mathcal{M}}

\def\pibehavior{\beta}  

\def\visitpi{d_\pi}

\def\visitbehavior{d_{\pibehavior}}

\def\Sset{S}
\def\Aset{A}

\renewcommand{\widehat}{\hat}

\def\eqref#1{equation~\ref{#1}}

\def\1{\bm{1}}

\DeclareMathAlphabet{\mathsfit}{\encodingdefault}{\sfdefault}{m}{sl}
\SetMathAlphabet{\mathsfit}{bold}{\encodingdefault}{\sfdefault}{bx}{n}

\definecolor{darkgreen}{rgb}{0,0.5,0}

\usepackage{amsmath,amsfonts,bm}

\def\eqref#1{equation~\ref{#1}}

\def\1{\bm{1}}

\DeclareMathAlphabet{\mathsfit}{\encodingdefault}{\sfdefault}{m}{sl}
\SetMathAlphabet{\mathsfit}{bold}{\encodingdefault}{\sfdefault}{bx}{n}

\newcommand{\E}{\mathbb{E}}

\newcommand{\R}{\mathbb{R}}

\newcommand{\kl}[2]{D_{\mathrm{KL}}(#1 \Vert #2)}

\newcommand{\mix}[1][(s,a)]{\lambda#1}
\newcommand{\adv}[1][\beta]{A_{#1}}

\DeclareMathOperator*{\argmax}{arg\,max}

\title{\BCRPO{}: Batch Residual Policy Optimization}

\iftrue
\author{
Sungryull Sohn\footnote{Equal contribution}$^{,1,2}$\and
Yinlam Chow$^{*,1}$\and
Jayden Ooi$^{*,1}$\and\\
Ofir Nachum$^1$\and
Honglak Lee$^{1,2}$\and
Ed Chi$^{1}$\And
Craig Boutilier$^1$
\affiliations
$^1$Google Research\\
$^2$University of Michigan
\emails
srsohn@umich.edu, \{yinlamchow,\ jayden,\ ofirnachum,\ honglak,\ edchi,\ cboutilier\}@google.com
}
\fi

\begin{document}
\maketitle
\begin{abstract}
In batch reinforcement learning (RL), 
one often constrains a learned policy to be close to the behavior (data-generating)
policy,
e.g.,
by constraining
the learned action distribution to differ from the behavior policy by some maximum degree that
is the \emph{same at each state}.
This can cause batch RL to be overly conservative, unable
  to exploit large policy changes at frequently-visited, high-confidence states without risking poor performance at sparsely-visited states.
To remedy this, we propose 
\emph{residual policies}, where the allowable deviation of the learned policy
is \emph{state-action-dependent}.
We derive a new for RL method,  {\em \BCRPO{}}, which learns both the policy and allowable deviation 
that jointly maximize a lower bound on policy performance.
We show that \BCRPO{} achieves the state-of-the-art performance in a number of tasks.
\end{abstract}

\section{Introduction}\label{sec:intro}
Deep reinforcement learning (RL) methods are
increasingly successful in domains such as games \citep{mnih2013playing}, recommender systems \citep{gauci2018horizon}, and robotic manipulation \citep{nachum2019multi}. 
Much of this success relies on the ability to collect new data through \emph{online} interactions with the environment during training, often relying on simulation. Unfortunately, this approach is impractical in many real-world applications where faithful simulators are rare, and in
which active data collection through interactions with the environment is costly, time consuming, and risky.

\emph{Batch (or offline) RL}~\citep{lange2012batch} is an emerging research direction that aims to circumvent the need for
online data collection, instead learning a new policy using only offline trajectories generated by some behavior policy
(e.g., the currently deployed policy in some application domain).
In principle, any off-policy RL algorithm (e.g., DDPG \citep{lillicrap2015continuous}, DDQN \citep{van2016deep}) may be used in this batch (or more accurately, ``offline'')
fashion; but in practice, such methods have been shown to fail to learn when presented with arbitrary, static, off-policy data.
This can arise for several reasons: lack of exploration \citep{lange2012batch}, generalization error on out-of-distribution samples in value estimation \citep{kumar2019stabilizing}, or high-variance policy gradients induced by covariate shift \citep{mahmood2014weighted}. 

Various techniques have been proposed to address these issues, many of which can be interpreted as constraining or regularizing the learned policy to be \emph{close} to the behavior policy
\citep{fujimoto2018off,kumar2019stabilizing} (see further discussion below). 
While these batch RL methods show promise,
none provide improvement guarantees relative to the behavior policy. In domains for which batch RL is well-suited (e.g., due to the risks of active data collection), such guarantees can be critical to deployment of the resulting RL policies.

In this work, we use the well-established methodology of {\em conservative policy improvement (CPI)} \citep{kakade2002approximately} to develop a theoretically principled use of behavior-regularized RL in the batch setting.
Specifically, we parameterize the learned policy as a {\em residual policy}, in which a base (behavior) policy is combined linearly with a learned {\em candidate policy}  using a mixing factor called the {\em confidence}.
Such residual policies are motivated by several practical considerations. First, one often has access
to offline data or logs generated by a deployed base policy which is known to perform reasonably well. The offline
data can be used by an RL method to learn a candidate policy with better \emph{predicted} performance, but if confidence in 
parts of that prediction is weak, relying on the base policy may be desirable. The base policy may also
incorporate soft business constraints or some form of interpretability.
Our residual policies blend the two in a learned,
non-uniform fashion.
When deploying a new policy, we use the CPI framework to derive updates 
that learn both the candidate policy and the confidence that jointly maximize a lower bound on performance improvement relative to the behavior policy.
Crucially, while traditional applications of CPI, such as TRPO~\citep{schulman2015trust}, use a constant or state-independent confidence, our performance bounds and learning rules are based on {\em state-action-dependent} confidences---this gives rise to bounds that are less conservative than their CPI counterparts.

In Sec.~2, we formalize residual policies and in Sec.~3 analyze a novel
\emph{difference-value function}.
Sec.~4 holds our main result, a tighter lower bound on policy improvement for our residual approach (vs.\ CPI and TRPO). We derive the \BCRPO{} algorithm in Sec.~5 to jointly learn the candidate policy and confidence; experiments in Sec.~6 show its effectiveness. 

\section{Preliminaries}\label{sec:prelim}
We consider a \emph{Markov decision process (MDP)}
$\mdp = \langle \Sset, \Aset, R, T,  P_0 \rangle$, with state space $\Sset$, action space $\Aset$, reward function
$R$, transition kernel $T$, and initial state distribution $P_0$.
A policy $\pi$ interacts with the environment, starting at $s_0 \sim  P_0$.  At step $t$, the policy samples an action $a_t$ from a distribution $\pi(\cdot|s_t)$ over $\Aset$ and applies. 
The environment emits a reward $r_t= R(s_t, a_t)\in[0,R_{\max}]$ and next state $s_{t+1} \sim T(\cdot|s_t, a_t)$.
In this work, we consider discounted infinite-horizon problems with discount factor $\gamma\in[0,1)$. 

Let
\begin{small}$\Delta=\{\pi:\Sset\times \Aset\rightarrow [0,1],\;\sum_{a}\pi(a|s)=1\}$\end{small}
be the set of Markovian stationary policies. 
The expected (discounted) cumulative return of policy \begin{small}$\pi\in \Delta$\end{small}, is
$J_\pi:=\mathbb E_{T,\pi}[\sum_{t=0}^{\infty}\gamma^t R(s_t,a_t)\mid s_0\sim P_0]$.
Our aim is to find an optimal policy
$\pi^*\in\argmax_{\pi\in \Delta} \,J_\pi$.
In reinforcement learning (RL), we must do so without knowledge of $R, T$, using only trajectory data generated
from the environment (see below) or access to a simulator (see above).

We consider pure offline or \emph{batch RL}, where
the learner has access to a fixed data set (or \emph{batch}) of state-actions-reward-next-state samples
$B = \{(s, a, r, s')\}$, generated by a (known) \emph{behavior policy} $\pibehavior(\cdot|s)$.
No additional
data collection
is permitted. We denote by $d_{\pibehavior}$ the $\gamma$-discounted occupation measure of the MDP w.r.t.\ $\pibehavior$.

In this work, we study the problem of \emph{residual policy optimization (RPO)} in the batch setting.
Given the behavior policy $\pibehavior(a|s)$, we would like to learn a \emph{candidate policy} $\rho(a|s)$ and a state-action \emph{confidence} $\mix$, such that the final \emph{residual policy}
$\pi(a|s)=(1 - \mix)\cdot\pibehavior(a|s) + \mix\cdot\rho(a|s)$ maximizes total return. As discussed above,
this type of mixture allows one to exploit an existing, ``well-performing'' behavior policy. Intuitively, $\mix$ should capture how much we can trust $\rho$ at each $s, a$ pair, given the available data. To ensure that the residual policy is a probability distribution at every state $s\in \Sset$, we constrain the confidence $\lambda$ to lie in the set
$
 \Lambda(s)=\left\{\lambda: \Sset \times \Aset \rightarrow[0,1]:\sum_{a}\mix\left(\pibehavior(a|s)-\rho(a|s)\right)=0\right\}.
$

\paragraph{Related Work.}
Similar to the above policy formulation, CPI \citep{kakade2002approximately} also develops a policy mixing methodology that guarantees performance improvement when the confidence $\lambda$ is a constant. However, CPI is an online algorithm, and it learns the candidate policy independently of (not jointly with) the mixing factor; thus, extension of CPI to offline, batch setting is unclear.
Other existing work also deals with online residual policy learning without jointly learning mixing factors
\citep{johannink2019residual, silver2018residual}.
Common applications of CPI
may treat $\lambda$ as a hyper-parameter, which specifies the maximum total-variation distance
between the learned and behavior policy distributions (see standard proxies in
\citet{schulman2015trust,pirotta2013safe} for details).
 
\emph{Batch-constrained Q-learning (BCQ)} \citep{fujimoto2018off, fujimoto2019benchmarking} incorporates the behavior policy 
when defining
the admissible action set in Q-learning for selecting the highest-valued actions that are similar to data samples in the batch. 
\emph{BEAR}~\citep{kumar2019stabilizing} is motivated as a means to control the accumulation of out-of-distribution value errors; but its main algorithmic contribution 
is realized by adding a regularizer to the loss that measures the kernel maximum mean discrepancy (MMD) \citep{gretton2007kernel} between the learned and behavior policies similar to KL-control~\citep{jaques2019way}.
Algorithms such as SPI \citep{ghavamzadeh2016safe} and SPIBB \citep{laroche2017safe} bootstraps the learned policy with the behavior policy when the uncertainty in the update for current state-action pair is high, where the uncertainty is measured by the visitation frequency of state-action pairs in the batch data. While these methods work well in some applications it is unclear if they have any performance guarantees.

\section{The Difference-value Function}\label{sec:rpo}


We begin by defining and characterizing the \emph{difference-value function},
a concept we exploit in the derivation of our batch RPO method in Secs.~4 and~5.
For any $s\in S$, let $V_\pi(s)$ and $V_\pibehavior(s)$ be the value functions induced by policies $\pi$ and $\pibehavior$, respectively. 
Using the structure of the residual policy,  we establish two characterizations of the \emph{difference-value function} $\Delta V_{\pi,\pibehavior}(s):=V_\pi(s)-V_\pibehavior(s)$.
\begin{lemma}\label{lem:main}
Let $A_\pi(s,a):=Q_\pi(s,a)-V_\pi(s)$ be the advantage function w.r.t. residual policy $\pi$, where $Q_\pi$ is the state-action value.
The difference-value is
$
\Delta V_{\pi,\pibehavior}(s)\!=\!\mathbb E_{T,\pibehavior} [ \sum_{t=0}^\infty\gamma^t\Delta\widehat{A}_{\pi,\pibehavior,\rho,\lambda}(s_t)\!\mid\! s_0=s ],
$
where

\vspace{-0.1in}
\begin{small}
\[
\smash{\Delta\widehat{A}_{\pibehavior,\rho,\lambda}(s) = \sum_{a\in \Aset} \beta(a|s)\cdot\mix\cdot\frac{\rho(a|s) - \pibehavior(a|s)}{\pibehavior(a|s)}\cdot \adv[\pi](s, a)}
\]
\end{small}
\vspace{-0.1in}

is the residual reward that depends on $\lambda$ and difference of candidate policy $\rho$ and behavior policy $\pibehavior$. 
\end{lemma}
This result establishes that the difference value is essentially a value function w.r.t.\ the residual reward. Moreover, it is proportional to the advantage of the target policy, the confidence, and the difference of policies.
While the difference value can be estimated from behavior data batch $B$,  this formulation requires knowledge of the advantage function $A_\pi$ w.r.t.\ the target policy, which must be re-learned at every $\pi$-update in an off-policy fashion. Fortunately, we can show that the difference value can 
also be expressed as a function of the advantage w.r.t.\
the \emph{behavior policy} $\pibehavior$:
\begin{theorem}\label{thm:main}
Let $\adv(s,a):=Q_\pibehavior(s,a)-V_\pibehavior(s)$ be the advantage function induced by $\pibehavior$, in which $Q_\pibehavior$ is the state-action value. The difference-value is given by $\Delta V_{\pi,\pibehavior}(s)\!=\!\mathbb E_{T,\pi}\left[\sum_{t=0}^\infty\!\gamma^t\Delta A_{\pibehavior,\rho,\lambda}(s_t)\!\!\mid\!\! s_0=s\right]$,
where

\vspace{-0.1in}
\begin{small}
\[
\smash{
\Delta A_{\pibehavior,\rho,\lambda}(s)=\sum_{a\in A} \pibehavior(a|s)\cdot\lambda(a|s)\cdot\frac{\rho(a|s)-\pibehavior(a|s)}{\pibehavior(a|s)}\cdot \adv(s,a)
}
\]
\end{small}
\vspace{-0.1in}

is the residual reward that depends on $\lambda$ and difference of candidate policy $\rho$ and behavior policy $\pibehavior$. 
\end{theorem}

In our RPO approach, we exploit the
nature of the difference-value function to solve the maximization w.r.t.\ the confidence and candidate policy: $(\mix[]^*(s, \cdot),\rho^*(\cdot|s))\!\in \argmax_{\lambda\in \Lambda(s),\rho\in\Delta}\Delta V_{\pi}(s)$, $\forall s\in \Sset$. Since $\mix[](s, \cdot)=0$ implies $\Delta V_{\pi,\pibehavior}(s)=0$, the optimal difference-value function $\Delta V^*(s):=\max_{\lambda\in \Lambda(s),\rho\in\Delta}\Delta V_{\pi}(s)$ is always lower-bounded by $0$.
We motivate computing $(\lambda,\rho)$ with the above difference-value formulation rather than as a standard RL problem as follows. In the tabular case, optimizing $(\lambda,\rho)$ with either formulation gives an identical result. However, both the difference-value function in Theorem~\ref{thm:main} and the standard RL objective require sampling data generated by the updated policy $\pi$.
In the batch setting, when fresh samples are unavailable, learning $(\lambda,\rho)$ with off-policy data 
may incur instability due to high generalization error \citep{kumar2019stabilizing}. While this can be alleviated by adopting the CPI
methodology, applying CPI directly to RL can be overly conservative \citep{schulman2015trust}.
By contrast, we leverage the special structure of the difference-value function (e.g., non-negativity) below, using
this new formulation together with CPI to derive a less conservative RPO algorithm.

\section{Batch Residual Policy Optimization}\label{sec:bcrpo}
We now develop an RPO algorithm that has stable learning performance in the batch setting and \emph{performance improvement guarantees}. For the sake of brevity, in the following we only present the main results on performance guarantees of RPO. Proofs of these results can be found in the appendix of the extended paper. 
We begin with
the following baseline result, directly applying Corollary~1 of the TRPO result to RPO to
ensure the residual policy $\pi$ performs no worse than $\pibehavior$.
\begin{lemma}\label{lem:cpi_1}
For any value function $U:S\to\R$, the difference-return satisfies
$J_\pi - J_\beta \ge
    \frac{1}{1-\gamma} \widetilde L_{U,\pibehavior,\rho,\lambda} - \frac{2 \gamma}{(1-\gamma)^2} \cdot \epsilon_{U,\pibehavior,\rho,\lambda} \cdot \E_{s\sim \visitbehavior}\left[
    \sqrt{\frac{1}{2}
    \kl{\pibehavior(s)}{\rho(s)}
    }
    \right],
$
where the surrogate objective and the penalty weight are
\begin{small}
\begin{equation*}
\begin{split}
&\widetilde L_{U,\pibehavior,\rho,\lambda} \!:=\! \mathop{\E}_{(s,a,s')\sim \visitbehavior}\left[
    \mix \cdot\frac{\rho(a|s)-\pibehavior(a|s)}{\pibehavior(a|s)}\!\cdot\!
    \Delta U(s,a,s')
    \right],\\
    &\epsilon_{U,\pibehavior,\rho,\lambda} \!:=\! \max_{s}\! |\E_{\pi, T}[\Delta U(s,a,s')]|,
\end{split}
\end{equation*}
\end{small}
where $\Delta U(s,a,s') := R(s,a) + \gamma U(s') - U(s)$.
\end{lemma}

When $U = V_\pi$, one has $\E_{a\sim\pi(s)}[\Delta U(s,a,s')]=0$, $\forall s\in \Sset$, which implies that the inequality is tight---this lemma then
coincides Lemma~\ref{lem:main}.  While this CPI result forms the basis of many RL algorithms (e.g., TRPO, PPO), in many cases it is very loose since $\epsilon_{U,\beta,\rho,\lambda}$ is a maximum over all states. Thus, using this bound for policy optimization may be \emph{overly conservative}, i.e., algorithms which rely on this bound must take very small policy improvement steps, especially when the penalty weight $\epsilon_{U,\pibehavior,\rho,\lambda}$ is large, i.e., $|\epsilon_{U,\pibehavior,\rho,\lambda} / (1-\gamma)|>>|\widetilde L_{U,\pibehavior,\rho,\lambda}|$.   While this approach may be reasonable in online settings---when collection of new data (with an updated behavior policy $\pibehavior\leftarrow\pi$) is allowed---in the batch setting
it is challenging to overcome such conservatism.

To address this issue, we develop a CPI method that is specifically tied to the difference-value formulation, and uses a state-action-dependent confidence $\mix$. We first derive the following theorem, which bounds the difference returns that are generated by $\pibehavior$ and $\pi$.
\begin{theorem}\label{thm:cpi_2}
The difference return of $(\pi,\pibehavior)$ satisfies
\begin{equation*}
    J_\pi \!-\! J_\beta \!\ge\! \frac{1}{1-\gamma}\left( L^{\prime}_{\pibehavior,\rho,\lambda} \!-\! \frac{\gamma}{1-\gamma}\!\cdot\! L^{\prime\prime}_{\pibehavior,\rho,\lambda}\!\cdot \!\max_{s_0\in S}\,L^{\prime\prime\prime}_{\pibehavior,\rho,\lambda}(s_0)
    \right)\!,
\end{equation*}
where the surrogate objective function, regularization, and penalty weight are given by
\begin{small}
\begin{equation*}
\begin{split}
    L^{\prime}_{\pibehavior,\rho,\lambda}:=& \E_{(s,a)\sim \visitbehavior}\left[
    \mix \cdot\frac{\rho(a|s)-\pibehavior(a|s)}{\pibehavior(a|s)} \cdot
    A_\pibehavior(s,a)
    \right]\\
    L^{\prime\prime}_{\pibehavior,\rho,\lambda}:=& \E_{(s,a)\sim \visitbehavior}\left[
    \mix \cdot\frac{|\rho(a|s)-\pibehavior(a|s)|}{\pibehavior(a|s)}\right]\\
    L^{\prime\prime\prime}_{\pibehavior,\rho,\lambda}(s_0):=& \E_{(s,a)\sim \visitbehavior(s_0)}\! \left[
    \mix\!\cdot\!\frac{|\rho(a|s)\!-\!\pibehavior(a|s)|}{\pibehavior(a|s)}\!\cdot\! |A_\pibehavior(s,a)| \right]
\end{split}
\end{equation*}
\end{small}
respectively, in which $\visitbehavior(s_0)$ is the discounted occupancy measure w.r.t. $\pibehavior$ given initial state $s_0$.
\end{theorem}

Unlike the difference-value formulations in Lemma \ref{lem:main} and Theorem \ref{thm:main}, which require the knowledge of advantage function $A_\pi$ or the trajectory samples generated by $\pi$, the lower bound in Theorem \ref{thm:cpi_2} is comprised only of terms that can be estimated directly
using the data batch $B$ (i.e., data generated by $\beta$). This makes it a natural objective function for batch RL. Notice also that the surrogate objective, the regularization, and the penalty weight in the lower bound are each
proportional to the confidence and to the relative difference of the candidate and behavior policies. 
However, the $\max$ operator requires state enumeration to compute this lower bound, which is intractable when $S$ is large or uncountable. 

We address this by introducing a slack variable $\kappa\geq 0$ to replace the $\max$-operator with suitable constraints. This allows the bound on the difference return to be rewritten as:
$
J_\pi - J_\beta \ge \frac{1}{1-\gamma} L^{\prime}_{\pibehavior,\rho,\lambda} -\min_{\kappa\geq L^{\prime\prime\prime}_{\pibehavior,\rho,\lambda}(s_0),\,\,\forall s_0} \frac{\gamma}{(1-\gamma)^2} L^{\prime\prime}_{\pibehavior,\rho,\lambda} \cdot \kappa.
$
Consider the Lagrangian of the lower bound:
\begin{small}
\begin{equation*}
\begin{split}
    \frac{L^{\prime}_{\pibehavior,\rho,\lambda}}{1-\gamma}
     - \min_{\kappa\geq 0}&\max_{\eta(s)\geq 0,\forall s}\frac{\gamma \cdot L^{\prime\prime}_{\pibehavior,\rho,\lambda} \cdot \kappa}{(1-\gamma)^2} -\sum_{s}\eta(s)(\kappa- L^{\prime\prime\prime}_{\pibehavior,\rho,\lambda}(s)).
    \end{split}
\end{equation*}
\end{small}
To simplify this saddle-point problem, we restrict the Lagrange multiplier to be $\eta(s)=\eta\cdot P_0(s)\geq 0$, where $\eta\geq 0$ is a scalar multiplier. Using this approximation and the strong duality of linear programming \citep{boyd2004convex} over primal-dual variables $(\kappa,\eta)$, the saddle-point problem on $(\lambda,\rho,\eta,\kappa)$ can be re-written as
\begin{small}
\begin{align}
  \mathcal{L}_{\pibehavior,\rho,\lambda}:=&\max_{\eta\geq 0}\min_{\kappa\geq 0}  \frac{L^{\prime}_{\pibehavior,\rho,\lambda} -  L^{\prime\prime}_{\pibehavior,\rho,\lambda}\cdot\kappa\cdot\frac{\gamma}{1-\gamma}-\eta\cdot\kappa+\eta L^{\prime\prime\prime}_{\pibehavior,\rho,\lambda}}{1-\gamma}\nonumber\\
    =&\frac{1}{1-\gamma}\left( 
    L^{\prime}_{\pibehavior,\rho,\lambda} -  \frac{\gamma}{1-\gamma} L^{\prime\prime}_{\pibehavior,\rho,\lambda}\cdot L^{\prime\prime\prime}_{\pibehavior,\rho,\lambda}
    \right),\label{bdd:saddle_point_approx}
    \end{align}
\end{small}
where $L^{\prime\prime\prime}_{\pibehavior,\rho,\lambda}=\mathbb E_{s\sim P_0}[L^{\prime\prime\prime}_{\pibehavior,\rho,\lambda}(s)]$. The equality is based on the KKT condition on $(\kappa,\eta)$. Notice that the only difference between the CPI lower bound in Theorem \ref{thm:cpi_2} and the objective function $\mathcal{L}_{\pibehavior,\rho,\lambda}$ is that the $\max$ operator is replaced by expectation w.r.t\ the initial distribution. 

With certain assumptions on the approximation error of the Lagrange multiplier parametrization $\eta(s)\approx P_0(s)$, we can characterize the gap between the original CPI objective function in Theorem~\ref{thm:cpi_2} and $\mathcal{L}_{\pibehavior,\rho,\lambda}$. 
One approach is to look into the KKT condition of the original saddle-point problem and bound the sub-optimality gap introduced by this Lagrange parameterization. Similar derivations can be found in the analysis of approximate linear programming (ALP) algorithms \citep{abbasi2019large,de2003linear}.

Compared with the vanilla CPI result from Lemma \ref{lem:cpi_1}, there are two characteristics in problem (\ref{bdd:saddle_point_approx}) that make the optimization w.r.t.\ $\mathcal{L}_{\pibehavior,\rho,\lambda}$ less conservative. First, the penalty weight $L^{\prime\prime\prime}_{\pibehavior,\rho,\lambda}$ here is smaller than $\epsilon_{U,\pibehavior,\rho,\lambda}$ in Lemma~\ref{lem:cpi_1}, which means that the corresponding objective has less incentive to force $\rho$ to be close to $\pibehavior$.
Second, compared with entropy regularization in vanilla CPI, here the regularization and penalty weight are both linear in $\lambda\in \Lambda\subseteq[0,1]^{|A|}$; thus, unlike vanilla CPI, whose objective  is linear in $\lambda$,
our objective is quadratic in $\lambda$---this modification ensures the optimal value is not
a \emph{degenerate extreme point} of $\Lambda$.\footnote{For example, when $\lambda$ is state-dependent (which automatically satisfies the equality constraints in $\Lambda$), the linear objective in vanilla CPI makes the optimal value $\lambda^*(\cdot|s)$ a \emph{0-1 vector}. Especially when $\frac{2 \gamma}{1-\gamma}\E_{s\sim \visitbehavior}\left[
    \sqrt{\frac{1}{2}
    \kl{\pibehavior(s)}{\rho(s)}
    }\,\right]$ is large, then most entries of $\lambda^*$ become zero, i.e., $\pi$ will be very close to $\pibehavior$.}

\section{The BRPO Algorithm}\label{sec:algorithm}

We now develop the BRPO algorithm, for which the general pseudo-code is given in Algorithm \ref{alg:train}. Recall
that if the candidate policy $\rho$ and confidence $\lambda$ are jointly optimized
\begin{small}
\begin{equation}\label{problem:crpo}
  (\rho^*,\lambda^*)\!\in\!\argmax_{\lambda\in\Lambda,\,\,\rho\in\Delta} 
   \mathcal L_{\pibehavior,\rho,\lambda},
\end{equation}
\end{small}
then the residual policy $\pi^*(a|s)= (1-\lambda^*(s, a))\pibehavior(a|s)+\lambda^*(s ,a)\rho^*(a|s)$  performs no worse than behavior policy $\pibehavior$. Generally, solutions for problem (\ref{problem:crpo}) use a form of minorization-maximization (MM) \citep{hunter2004tutorial}, a class of methods that
also includes expectation maximization. In the terminology of MM algorithms, $\mathcal L_{\pibehavior,\rho,\lambda}$ is a surrogate function satisfying the following \emph{MM properties}:
\begin{equation}\label{eq:MM}
J_\pi-J_\pibehavior\geq \mathcal L_{\pibehavior,\rho,\lambda},\,\,\,
J_\pibehavior-J_\pibehavior\!=\! \mathcal L_{\pibehavior,\beta,\lambda}\!=\! \mathcal L_{\pibehavior,\rho,0}=0, 
\end{equation}
which guarantees that it minorizes the difference-return $J_\pi-J_\pibehavior$ with equality at $\lambda=0$ (with arbitrary $\rho$) or at $\rho=\beta$ (with arbitrary $\lambda$). This algorithm is also reminiscent of proximal gradient methods.
We optimize
$\lambda$ and $\rho$ in RPO with a simple two-step \emph{coordinate-ascent}. Specifically, at iteration $k\in\{0,1,\ldots,K\}$, given confidence $\lambda_{k-1}$, we first compute an updated candidate policy $\rho_{k}$, and with $\rho_{k}$ fixed, we update $\lambda_k$, i.e., $\mathcal L_{\pibehavior,\rho_{k},\lambda_k}\geq \mathcal L_{\pibehavior,\rho_{k},\lambda_{k-1}}\geq \mathcal L_{\pibehavior,\rho_{k-1},\lambda_{k-1}}$.  When $\lambda$ and $\rho$ are represented tabularly or with linear function approximators, under certain regularity assumptions (the  Kurdyka-Lojasiewicz property \citep{xu2013block}) coordinate ascent guarantees global convergence (to the limit point) for BRPO.

However, when more complex representations (e.g., neural networks) are used to parameterize these decision variables, this property no longer holds. While one may still compute $(\lambda^*,\rho^*)$  with first-order methods (e.g., SGD), convergence to local optima is not guaranteed. To address this, we next further restrict the MM procedure to develop closed-form solutions for both the candidate policy
and the confidence.

\paragraph{The Closed-form Candidate Policy $\rho$.}
To effectively update the candidate policy when given the confidence $\lambda\in\Lambda$, we
develop a closed-form solution for $\rho$. Our approach is based on maximizing the following objective, itself a more conservative version of the CPI lower bound $\mathcal L_{\pibehavior,\rho,\lambda}$:
\begin{small}
\begin{align}
&\max_{\rho\in\Delta}
    \,\widehat{\mathcal L}_{\pibehavior,\rho,\lambda}
    :=\E_{s\sim \visitbehavior}\!\bigg[\!
    \E_{a\sim\pibehavior}\left[\mix \frac{\rho(a|s)-\pibehavior(a|s)}{\pibehavior(a|s)} 
    A_{\pibehavior}\right]\nonumber\\
    &\quad- \frac{\gamma\max\{\kappa_{\lambda}(s),\kappa_{|A_{\pibehavior} |\lambda}(s)\}}{2(1-\gamma)}\cdot
    \kl{\rho}{\pibehavior}(s)
    \bigg]\cdot \frac{1}{1-\gamma},\label{problem:crpo_mu_2}
    \end{align}
\end{small}
where $\kappa_g(s)= (1+ \log \mathbb E_\pibehavior[\exp(g(a|s)^2)])>0$ for any arbitrary non-negative function $g$. To show that 
$\widehat{\mathcal L}_{\pibehavior,\rho,\lambda}$ in (\ref{problem:crpo_mu_2}) is an eligible lower bound (so that the corresponding $\rho$-solution is an MM), we need to show that it satisfies the properties in (\ref{eq:MM}). When $\rho=\pibehavior$, by the definition of $\widehat{\mathcal L}_{\pibehavior,\rho,\lambda}$ the second property holds. To show the first property, we first consider the following problem:
\begin{small}
\begin{equation}\label{problem:crpo_mu_ori}
  \max_{\rho\in\Delta}\frac{1}{1-\gamma}\left(
    L^{\prime}_{\pibehavior,\rho,\lambda}  - \frac{\gamma}{1-\gamma} \widetilde L^{\prime\prime}_{\pibehavior,\rho,\lambda}\cdot \widetilde L^{\prime\prime\prime}_{\pibehavior,\rho,\lambda}\right),
\end{equation}
\end{small}
where $ L^{\prime}_{\pibehavior}(\rho,\lambda)$ is given in Theorem \ref{thm:cpi_2}, and
\begin{small}
\[
\begin{split}
\widetilde L^{\prime\prime}_{\pibehavior,\rho,\lambda}=&\E_{s\sim \visitbehavior}\left[\sqrt{\kappa_\lambda(s)\cdot \kl{\rho}{\pibehavior}(s)/2}\right], \\
\widetilde L^{\prime\prime\prime}_{\pibehavior,\rho,\lambda}=& \E_{s\sim \visitbehavior} \left[\sqrt{{ \kappa_{|A_{\pibehavior}| \lambda}(s)\cdot \kl{\rho}{\pibehavior}(s)}/{2}} \right].
\end{split}
\]
\end{small}
The concavity of $\sqrt{(\cdot)}$ (i.e., $\mathbb E_{s\sim \visitbehavior}[\sqrt{(\cdot)}]\leq \sqrt{\mathbb E_{s\sim \visitbehavior}[(\cdot)]}$) and monotonicity of expectation
imply that the objective in (\ref{problem:crpo_mu_2}) is a lower bound of that in (\ref{problem:crpo_mu}) below. Furthermore, by the weighted Pinsker's inequality \citep{bolley2005weighted}
$\sum_a|g(a|s)(\rho(a|s)-\beta(a|s))|\leq \sqrt{\kappa_g(s)
\kl{\rho}{\pibehavior}(s)/2}$, we have: 
(i) $0\leq \widetilde L^{\prime\prime}_{\pibehavior,\rho,\lambda}\leq L^{\prime\prime}_{\pibehavior,\rho,\lambda}$;
and (ii) $0\leq \widetilde L^{\prime\prime\prime}_{\pibehavior,\rho,\lambda}\leq L^{\prime\prime\prime}_{\pibehavior,\rho,\lambda}$, which implies the objective in (\ref{problem:crpo_mu_ori}) is a lower-bound of that in (\ref{problem:crpo}) and validates the first MM property.

Now recall the optimization problem: $\max_{\rho\in\Delta}
    \,\widehat{\mathcal L}_{\pibehavior,\rho,\lambda}$.
    Since this optimization is over the state-action mapping $\rho$, the Interchangeability Lemma \citep{shapiro2009lectures} allows swapping the order of $\E_{s\sim \visitbehavior}$ and $\max_{\rho\in\Delta}$.
    This implies that at each $s\in S$ the candidate policy can be solved using:
\begin{small}
\begin{align}
\rho^*_\lambda\!\in\!&\argmax_{\rho(\cdot|s)\in\Delta}
   \mathbb E_{a\sim\pibehavior}\!\left[\left(\! \lambda\cdot\frac{\rho-\pibehavior}{\pibehavior}\cdot A_{\pibehavior} \right)\!(s,a) - \tau_\lambda(s)\log \frac{\rho(a|s)}{\pibehavior(a|s)}\right]\nonumber\\
=&\argmax_{\rho(\cdot|s)\in\Delta}
   \mathbb E_{a\sim\rho}\!\left[ \mix A_{\pibehavior} \!-\! \tau_\lambda(s)\log \frac{\rho(a|s)}{\pibehavior(a|s)}\right],\label{problem:crpo_mu}
\end{align}
\end{small}
where $\tau_\lambda(s)={\gamma\max\{\kappa_\lambda(s),\kappa_{|A_{\pibehavior}|\lambda}(s)\}}/{(2-2\gamma)}$ is the state-dependent penalty weight of the relative entropy regularization.
By the KKT condition of~(\ref{problem:crpo_mu}),  the optimal candidate policy $\rho^*_\lambda$ has the form
\begin{small}
\begin{equation}\label{eq:relateive_soft}
\rho^*_{\lambda}(a|s)=\frac{\pibehavior(a|s)\cdot\exp\left(\frac{\mix A_{\pibehavior}}{\tau_\lambda(s)}\right)}{\mathbb E_{a'\sim\pibehavior}[\exp({\mix[(s, a')] A_{\pibehavior}(s,a')}/{\tau_{\lambda}(s)})]}.
\end{equation}
\end{small}
Notice that the optimal candidate policy is a \emph{relative softmax policy}, which is a common solution policy for many entropy-regularized RL algorithms 
\citep{haarnoja2018soft}.
Intuitively, when the mixing factor vanishes (i.e., $\mix =0$), the candidate policy equals to the behavior policy, and with confidence we obtain the candidate policy by modifying the behavior policy $\pibehavior$ via \emph{exponential twisting}.



\paragraph{The Closed-form Confidence $\lambda$.}
Given candidate policy $\rho$, we derive efficient scheme for computing the confidence that solves the MM problem: $\max_{\lambda\in\Lambda} 
   \mathcal L_{\pibehavior,\rho,\lambda}$. Recall that this optimization can be reformulated as a concave quadratic program (QP) with linear equality constraints, which has a unique optimal solution \citep{faybusovich1997infinite}. However, since the decision variable (i.e., the confidence mapping) is infinite-dimensional, solving this QP is intractable without some assumptions about this mapping,  
   To resolve this issue, instead of using the surrogate objective $\mathcal L_{\pibehavior,\rho,\lambda}$ in MM, we turn to its sample-based estimate.
   Specifically, given a batch of data $B=\{(s_i,a_i,r_i,s'_i)\}_{i=1}^{|B|}$ generated by the behavior policy $\pibehavior$, denote by 
   \begin{small}
   \[
   \begin{split}
   &\overline{L}^{\prime}_{\pibehavior,\rho,\lambda}:=\frac{1}{1-\gamma}\cdot\frac{1}{|B|}\sum_{i=1}^{|B|}\overline\lambda_i^\top\cdot \left((\rho-\pibehavior)
    \cdot A_{\pibehavior}\right)_i\\
   &\overline{L}^{\prime\prime}_{\pibehavior,\rho,\lambda}:=\frac{1}{1-\gamma}\cdot\frac{1}{|B|}\sum_{i=1}^{|B|}\overline\lambda_i^\top\cdot |\rho-\pibehavior|_i\\
   &\overline{L}^{\prime\prime\prime}_{\pibehavior,\rho,\lambda}:=\frac{\gamma}{1-\gamma}\cdot\frac{1}{|B|}\sum_{i=1}^{|B|}\overline\lambda_i^\top\cdot \left(|\rho-\pibehavior|\cdot|A_{\pibehavior}|
   \right)_i
   \end{split}
   \]
   \end{small}
   the sample-average approximation (SAA) of functions ${L}^{\prime}_{\pibehavior,\rho,\lambda}$, ${L}^{\prime\prime}_{\pibehavior,\rho,\lambda}$, and ${L}^{\prime\prime\prime}_{\pibehavior,\rho,\lambda}$ respectively, where $((\rho-\pibehavior)
    \cdot A_{\pibehavior})=\{(\rho(\cdot|s_i)-\pibehavior(\cdot|s_i))
    \cdot A_{\pibehavior}(s_i,\cdot) \}_{s_i\in B}$,  $\left(|\rho-\pibehavior|
    \cdot |A_{\pibehavior} |\right)=\{(|\rho(\cdot|s_i)-\pibehavior(\cdot|s_i)| 
    \cdot |A_{\pibehavior}(s_i,\cdot)|)\}_{s_i\in B}$, and $|\rho-\pibehavior|=\{|\rho(\cdot|s_i)-\pibehavior(\cdot|s_i)|\}_{s_i\in B}$ are $|A|\cdot|B|$-dimensional vectors, where each element is generated by a state sample from $B$, and $\overline\lambda=\{\lambda(\cdot|s_i)\}_{s_i\in B}$ is a $|A|\cdot|B|$-dimensional decision vector, where each $|A|$-dimensional element vector corresponds to the confidence w.r.t. state samples in $B$.
   Since the expectation in ${L}^{\prime}_{\pibehavior,\rho,\lambda}$, ${L}^{\prime\prime}_{\pibehavior,\rho,\lambda}$, and ${L}^{\prime\prime\prime}_{\pibehavior,\rho,\lambda}$ is over the stationary distribution induced by the behavior policy, all the SAA functions are \emph{unbiased} Monte-Carlo estimates of their population-based counterparts. We now define $\overline{\mathcal L}_{\pibehavior,\rho,\lambda}:=\overline{L}^{\prime}_{\pibehavior,\rho,\lambda}- \overline{L}^{\prime\prime}_{\pibehavior,\rho,\lambda}\overline{L}^{\prime\prime\prime}_{\pibehavior,\rho,\lambda}$ as the SAA-MM objective and use this to solve for the confidence vector $\overline\lambda$ over the batch samples.
   
   Now consider the following maximization problem:
   \begin{small}\begin{equation}
  \max_{\overline\lambda\in\overline \Lambda}\langle(\rho-\pibehavior)
    \cdot A_{\pibehavior},\overline\lambda\rangle-\frac{\gamma}{|B|(1-\gamma)}\langle|\rho-\pibehavior|,\overline\lambda\rangle\cdot\langle|\rho-\pibehavior||A_{\pibehavior}|,\overline\lambda\rangle,
\end{equation}
\end{small}
where the feasible set  $\overline\Lambda=\{\lambda\in[0,1]:\sum_{a\in \Aset}\mix[(s_i, a)]\cdot (\rho(a|s_i)-\pibehavior(a|s_i))\cdot=0,\forall i\in \{1,\ldots,|B|\}\}$ only imposes constraints on the states that appear in the batch $B$.

This finite-dimensional QP problem can be expressed in the following quadratic form:
\begin{small}
\[
\max_{\overline\lambda\in\overline \Lambda}\,\,\overline \lambda^\top\Big((\rho-\pibehavior)
    \cdot A_{\pibehavior} \Big)-\frac{1}{2}\cdot\overline \lambda^\top \Theta\overline\lambda,
\]
\end{small}
where the symmetric matrix is given by
\begin{small}
\[
\Theta_{\pibehavior,\rho}:=\frac{\gamma (D_{|A_{\pibehavior}|}\cdot |\rho-\pibehavior|\cdot|\rho-\pibehavior|^\top+ |\rho-\pibehavior|\cdot |\rho-\pibehavior|^\top\cdot D_{|A_{\pibehavior}|}^\top)}{|B|(1-\gamma)},
\]
\end{small}
and
$D_{| A_{\pibehavior} |}=\text{diag}(\{|A_{\pibehavior} |\}_{a\in \Aset,s\in B})$ is a $|B|\cdot|\Aset|\times|B|\cdot|\Aset|$-diagonal matrix whose elements are the absolute advantage function. By definition, $\Theta$ is positive-semi-definite, hence the QP above is concave. 
Using its KKT condition, the unique optimal confidence vector over batch $B$ is given as
\begin{equation}\label{eq:lambda_vector}
\overline\lambda^* =\min\{1,\max\{0,\Theta_{\pibehavior,\rho}^{-1}((\rho-\pibehavior)
\cdot A_{\pibehavior} + M_{\pibehavior,\rho}^\top\nu_{\pibehavior,\rho}\}\},
\end{equation}
where $M_{\pibehavior,\rho}=\text{blkdiag}(\{[\rho(a|x)-\pibehavior(a|x)]_{a\in A}\}_{x\in B})$ is a $|B|\cdot |A|\times |B|$-matrix,  and the Lagrange multiplier $\nu_{\pibehavior,\rho}\in\mathbb{R}^{|B|}$ w.r.t. constraint $M_{\pibehavior,\rho}\overline \lambda=0$ is given by
\begin{equation}
\nu_{\pibehavior,\rho}\!=\!-(M_{\pibehavior,\rho}^\top\Theta^{-1}_{\pibehavior,\rho}M_{\pibehavior,\rho})^{-1}(M_{\pibehavior,\rho}^\top\Theta_{\pibehavior,\rho}^{-1}(\rho-\pibehavior)
\cdot A_{\pibehavior}).
\end{equation}
We first construct the confidence function $\lambda(s,a)$ from the confidence vector $\overline\lambda^*$ over $B$, 
in the following tabular fashion:
$$
\lambda(s,a)=\left\{\begin{array}{cl}
   \overline\lambda^*_{s,a}  &  \text{if $(s,a)\in B$}\\
    0 & \text{otherwise}
\end{array}\right. .
$$
While this construction preserves optimality w.r.t.\ the CPI objective (\ref{problem:crpo}), it may be overly conservative, because the policy equates to the behavior policy by setting $\lambda=0$ at state-action pairs that are not in $B$ (i.e., no policy improvement). To alleviate this conservatism, we propose to learn a confidence function that generalizes to out-of-distribution samples.

\begin{table*}[th!]
\scriptsize
\setlength\tabcolsep{1.5pt}
\centering
\begin{tabular}{|c|c|c |c|c|c |c|c|c|c|}
\hline
Environment-$\varepsilon$ & DQN & BRPO-C & BRPO (ours) & BCQ & KL-Q & SPIBB & BC & Behavior Policy\\
\hline
Acrobot-0.05 & \tb{-91.2 $\pm$ 9.1} & -94.6 $\pm$ 3.8 & \tb{-91.9 $\pm$ 9.0} & -96.9 $\pm$ 3.7 & -93.0 $\pm$ 2.6 & -103.5 $\pm$ 24.1 & -102.3 $\pm$ 5.0 & -103.9\\
Acrobot-0.15 & \tb{-83.1 $\pm$ 5.2} & -91.7 $\pm$ 4.0 & \tb{-86.1 $\pm$ 10.1} & -97.1 $\pm$ 3.3 & -92.1 $\pm$ 3.2 & -91.1 $\pm$ 44.8 & -113.1 $\pm$ 5.6 & -114.3\\
Acrobot-0.25 & \tb{-83.4 $\pm$ 3.9} & -91.2 $\pm$ 4.1 & \tb{-85.3 $\pm$ 4.8} & -96.7 $\pm$ 3.1 & -90.0 $\pm$ 2.9 & -86.0 $\pm$ 5.8 & -124.1 $\pm$ 7.0 & -127.2\\
Acrobot-0.50 & -84.3 $\pm$ 22.6 & -90.9 $\pm$ 3.4 & \tb{-83.7 $\pm$ 16.6} & \tb{-77.8 $\pm$ 13.5} & -84.5 $\pm$ 3.8 & -106.8 $\pm$ 102.7 & -173.7 $\pm$ 8.1 & -172.4\\
Acrobot-1.00 & -208.9 $\pm$ 174.8 & \tb{-156.8 $\pm$ 22.0} & \tb{-121.7 $\pm$ 10.2} & -236.0 $\pm$ 85.6 & -227.5 $\pm$ 148.1 & -184.8 $\pm$ 150.2 & -498.3 $\pm$ 1.7 & -497.3\\
\hline
CartPole-0.05 & 82.7 $\pm$ 0.5 & 220.8 $\pm$ 117.0 & \tb{336.3 $\pm$ 122.6} & 255.4 $\pm$ 11.1 & \tb{323.0 $\pm$ 13.5} & 28.8 $\pm$ 1.2 & 205.6 $\pm$ 19.6 & 219.1\\
CartPole-0.15 & 299.3 $\pm$ 133.5 & 305.6 $\pm$ 95.2 & \tb{409.9 $\pm$ 64.4} & 255.3 $\pm$ 11.4 & \tb{357.7 $\pm$ 84.1} & 137.7 $\pm$ 11.7 & 151.6 $\pm$ 27.5 & 149.5\\
CartPole-0.25 & 368.5 $\pm$ 129.3 & \tb{405.1 $\pm$ 74.4} & 316.8 $\pm$ 64.1 & 247.4 $\pm$ 128.7 & \tb{441.4 $\pm$ 79.8} & 305.2 $\pm$ 119.7 & 103.0 $\pm$ 20.4 & 101.9\\
CartPole-0.50 & 271.5 $\pm$ 52.0 & \tb{358.3 $\pm$ 114.1} & \tb{433.8 $\pm$ 93.5} & 282.5 $\pm$ 111.8 & 314.1 $\pm$ 107.0 & 310.4 $\pm$ 128.0 & 39.7 $\pm$ 5.1 & 37.9\\
CartPole-1.00 & 118.3 $\pm$ 0.3 & \tb{458.6 $\pm$ 51.5} & \tb{369.0 $\pm$ 42.3} & 194.0 $\pm$ 25.1 & 209.7 $\pm$ 48.4 & 147.1 $\pm$ 0.1 & 22.6 $\pm$ 1.5 & 21.9\\
\hline
LunarLander-0.05 & -236.4 $\pm$ 177.6 & 35.6 $\pm$ 61.7 & \tb{88.2 $\pm$ 32.0} & 81.5 $\pm$ 14.9 & \tb{84.4 $\pm$ 26.3} & -200.4 $\pm$ 81.7 & 75.8 $\pm$ 17.7 & 73.7\\
LunarLander-0.15 & -215.6 $\pm$ 140.4 & 79.6 $\pm$ 29.7 & \tb{103.9 $\pm$ 49.8} & 80.3 $\pm$ 16.8 & 61.4 $\pm$ 39.0 & \tb{86.1 $\pm$ 73.3} & 76.4 $\pm$ 16.6 & 84.9\\
LunarLander-0.25 & 2.5 $\pm$ 101.3 & 109.5 $\pm$ 40.7 & \tb{141.6 $\pm$ 11.0} & 83.5 $\pm$ 14.6 & 78.7 $\pm$ 48.8 & \tb{166.0 $\pm$ 90.6} & 57.9 $\pm$ 13.1 & 57.3\\
LunarLander-0.50 & -104.6 $\pm$ 68.3 & 42.5 $\pm$ 71.4 & \tb{101.0 $\pm$ 39.6} & -13.2 $\pm$ 44.9 & \tb{66.2 $\pm$ 78.0} & -134.6 $\pm$ 17.1 & -32.6 $\pm$ 6.5 & -36.0\\
LunarLander-1.00 & -65.6 $\pm$ 45.9 & \tb{53.5 $\pm$ 44.1} & \tb{81.8 $\pm$ 42.1} & -69.1 $\pm$ 44.0 & -139.2 $\pm$ 29.1 & -107.1 $\pm$ 94.4 & -177.4 $\pm$ 13.1 & -182.6\\
\hline
\end{tabular}
\vspace{-3mm}
\caption{The mean and st. dev. of average return with the best hyperparameter configuration (with the top-2 results boldfaced). 
Full training curves are given in 
Figure~\ref{fig:exp1_best_mean} in the appendix. For \BCRPO{}-C, the optimal confidence parameter is found by grid search.}
\label{table:exp1_best_mean_new}
\end{table*}

\paragraph{Learning the Confidence.}
Given a confidence vector $\overline\lambda^*$ corresponding to samples in batch $B$, we learn the confidence function $\lambda_\phi(s,a)$ in supervised fashion. To ensure that the  confidence function satisfies the constraint: $\lambda_\phi\in\Lambda$, i.e., $\sum_a\lambda_{\phi}(s,a)(\rho(a|s)-\pibehavior(a|s))=0$, $\lambda_{\phi}(s,a)\in[0,1]$, $\forall s,a$\footnote{If one restricts $\lambda_\phi$ to be only \emph{state}-dependent,  this constraint immediately holds.}, we parameterize it as
\begin{small}
\begin{equation}
\lambda_{\phi^*}(s,a):=\frac{\pi_{\phi^*}(a|s)-\pibehavior(a|s)}{\rho(a|s)-\pibehavior(a|s)},\,\,\forall (s,a)\in S\times A,
\end{equation}
\end{small}
where $\pi_\phi\in\Delta$ is a learnable policy mapping, such that $\min\{\pibehavior(a|s),\rho(a|s)\}\leq\pi_\phi(a|s)\leq\max\{\pibehavior(a|s),\rho(a|s)\}$, $\forall s,a$. We then learn $\phi$ via the following KL distribution-fitting objective
\citep{rusu2015policy}:
\begin{small}
\[
\min_{\phi}\frac{1}{B}\sum_{(s,a)\in B}\!\!\pi_\phi(a|s)\log\left(\frac{\pi_\phi(a|s)}{(1-\overline\lambda^*_{s,a})\pibehavior(a|s)+\overline\lambda^*_{s,a}\cdot\rho(a|s)}\right).
\]
\end{small}
While this approach learns 
$\lambda_\phi$ by generalizing the confidence vector to out-of-distribution samples, when $\pi_\phi$ is a NN, one challenge is to 
enforce the constraint: $\min\{\pibehavior(a|s),\rho(a|s)\}\leq\pi_\phi(a|s)\leq\max\{\pibehavior(a|s),\rho(a|s)\}$, $\forall s,a$. Instead, 
using an in-graph convex optimization NN \citep{amos2017optnet}, we parameterize $\lambda_\phi$ with a NN with the following \emph{constraint-projection layer} $\Phi:S\rightarrow A$ before the output:
\begin{align}
\Phi(s)\in&\arg\min_{\lambda\in \mathbb{R}^{|A|}}\,\frac{1}{2}\sum_{a\in A}\|\lambda_a-\widetilde\lambda^*_{s,a}\|^2,\nonumber\\
\text{s.t.}&\, \sum_a\lambda_a(\rho(a|s)-\pibehavior(a|s))=0,\,\,0\leq \lambda\leq 1,\label{eq:optimize_lambda}
\end{align}
where, at any $s\in S$, the $|A|$-dimensional confidence vector label $\{\widetilde\lambda^*_{s,a}\}_{a\in A}$ is equal to 
$\{\overline\lambda^*_{\overline{s},a}\}_{a\in A}$ chosen from the batch confidence vector $\overline\lambda^*$ such that $\overline{s}$ in $B$ is closest to $s$. Indeed, analogous to the closed-form solution 
in (\ref{eq:lambda_vector}), this projection layer has a closed-form QP formulation
with linear constraints:
$
\Phi(s)=\min\{1,\max\{0,\widetilde\lambda^*_{s,\cdot}+(\rho(\cdot|s)-\beta(\cdot|s))\cdot\mu_{\beta,\rho}\}\}
$,
where Lagrange multiplier $\mu_{\beta,\rho}$ is given by
$
\mu_{\pibehavior,\rho}\!=\!-{(\rho(\cdot|s)-\beta(\cdot|s))^\top\widetilde\lambda^*_{s,\cdot}}/{\|\rho(\cdot|s)-\beta(\cdot|s)\|^2}.
$

Although the $\rho$-update is theoretically justified, in practice,
when the magnitude of $\kappa_\lambda(s)$ becomes large
(due to the conservatism of the weighted Pinsker inequality),
the relative-softmax candidate policy (\ref{eq:relateive_soft}) may be too close to the behavior policy $\pibehavior$, impeding learning of the residual policy (i.e., $\pi\approx\pibehavior$). To avoid this in practice, we can upper bound the temperature, i.e., $\kappa_\lambda(s) \leftarrow\min\{\kappa_{\max},\kappa_\lambda(s)\}$, or introduce a weak temperature-decay schedule, i.e., $\kappa_\lambda(s) \leftarrow \kappa_\lambda(s) \cdot \epsilon^k$, with a tunable $\epsilon\in[0,1)$.

\vspace{-0.05in}
\begin{algorithm}[th!]
\caption{BRPO algorithm}\label{alg:train}
\begin{algorithmic}[1]
\Require{ $B$: batch data; Tunable parameter $\mu\in[0,1]$ }
\For{$t=1,\ldots,N$}
    \State Sample mini-batch of transitions $(s, a, r, d, s')\sim B$
    \State Compute $\overline\lambda^*$ from Eq.~(\ref{eq:lambda_vector})
    \State Update confidence $\phi^*$ by Eq.~(\ref{eq:optimize_lambda})
        \State Update candidate policy $\rho^*_{\lambda_{\phi^*}}$ by Eq.~(\ref{eq:relateive_soft})
    \State Construct target critic network $V_{\theta'}(s') \!:=\! (1-\mu)\mathbb{E}_{a'\sim\beta}[Q_{\theta'}(s',a')] + \mu \max_{a'} Q_{\theta'}(s',a')$
    
    \Comment{See section~\ref{sec:bcrpo} for the analysis in the case $\mu=1$}
    \State Update $\theta \leftarrow \argmax_{\theta}\frac{1}{2}\left(
        Q_{\theta}(s,a) - r - \gamma V_{\theta'}(s')\right)^2$
    \State Update target network: $\theta'\leftarrow \tau \theta + (1-\tau)\theta'$
\EndFor
\end{algorithmic}
\end{algorithm}
\vspace{-0.05in}


\section{Experimental Results}\label{sec:experiments}
To illustrate the effectiveness of \BCRPO{}, we compare against six baselines: DQN~\citep{mnih2013playing},
discrete BCQ~\citep{fujimoto2019benchmarking},
KL-regularized Q-learning (KL-Q)~\citep{jaques2019way},
SPIBB~\citep{laroche2017safe}, Behavior Cloning (BC)~\citep{kober2010imitation}, and \BCRPO{}-C, which is a simplified version of BRPO that uses a constant (tunable) parameter as confidence weight\footnote{For algorithms designed for online settings, we modify data collection to sample only from offline / batch data.}.
We do not consider ensemble models, thus do not include methods like BEAR~\citep{kumar2019stabilizing} among our baselines. CPI is also excluded since it is subsumed by \BCRPO{}-C with a grid search on the confidence. It is also generally inferior to \BCRPO{}-C because candidate policy learning does not optimize the performance of the final mixture policy. We evaluated on three discrete-action OpenAI Gym tasks~\citep{openaigym}: \cart{}, \luna{}, and \acro{}.



 
%

The behavior policy in each environment is trained using standard DQN until it reaches $75\%$ of optimal performance, similar to the process adopted in related work (e.g., \citet{fujimoto2018off}).
To assess how exploration and the quality of behavior policy affect learning, we generate five sets of data for each task by injecting different random exploration into the same behavior policy. Specifically, we add $\varepsilon$-greedy exploration  for $\varepsilon = 1$ (fully random), $0.5$, $0.25$, $0.15$, and $0.05$, generating $100K$ transitions each for batch RL training.

All models use the same architecture for a given environment---details (architectures, hyper-parameters, etc.) are described in the appendix of
the extended paper. While training is entirely offline, policy performance is evaluated online using the simulator, at every $1000$ training iterations.
Each measurement is the average return w.r.t.\ $40$ evaluation episodes and $5$ random seeds, and results are averaged over a sliding window of size $10$.

Table~\ref{table:exp1_best_mean_new} shows the average return of \BCRPO{} and the other baselines under the best hyper-parameter configurations in each task setting. Behavior policy performance decreases as $\varepsilon$ increases, as expected, and BC matches that very closely. DQN performs poorly in the batch setting.
Its performance improves as $\varepsilon$ increases from $0.05$ to $0.25$, due to increased state-action coverage, but as $\varepsilon$ goes higher ($0.5$, $1.0$), the state space coverage decreases again since the (near-) random policy is less likely to reach a state far away from the initial state.

Baselines like BCQ, KL-Q and SPIBB follow the behavior policy in some ways, and showing different performance characteristics over the data sets. The underperformance relative to \BCRPO{} is more prominent for very low or very high $\varepsilon$, suggesting deficiency due to overly conservative updates or following the behavior policy too closely, when \BCRPO{} is able to learn.


Since \BCRPO{} exploits the statistics of each $(s, a)$ pair in the batch data, it achieves
good performance in almost all scenarios, outperforming the baselines. 
%
The stable performance and robustness across various scenarios make \BCRPO{} an appealing
algorithm for batch/offline RL in real-world, where it is usually difficult to estimate the amount of exploration required prior to training, given access only to batch data.

\section{Concluding Remarks}
We have presented {\em Batch Residual Policy Optimization} (\BCRPO{}) for learning residual policies in batch RL settings.
Inspired by CPI, we derived learning rules for jointly optimizing both the candidate policy and 
\emph{state-action dependent} confidence mixture of a residual policy to maximize a conservative lower bound on policy performance.
\BCRPO{} is thus more exploitative in areas of state space that are well-covered by the batch data and more conservative in others.
While we have shown successful application of \BCRPO{} to various benchmarks, future work includes deriving finite-sample analysis of \BCRPO{},
and applying \BCRPO{} to more practical batch domains (e.g., robotic manipulation, recommendation systems).


\iftrue
{
    \appendix
    \onecolumn 
    \newpage
    \section{Proofs for Results in Section~\ref{sec:rpo}}
\subsection{Proof of Lemma \ref{lem:main}}
Before going into the derivation of this theorem, we first have the following technical result that studies the distance of the occupation measures that are induced by $\pibehavior$ and $\pi$.
\begin{lemma}\label{lemma:tech}
The following expression holds for any state-next-state pair $(s,s')$:
\[
\left((I-\gamma T_{\pi})^{-1} -(I-\gamma T_{\pibehavior})^{-1}\right)(s'|s)=-\gamma(I-\gamma T_{\pibehavior})^{-1}\Delta T_{\pibehavior,\rho,\lambda}(I-\gamma T_{\pi})^{-1}(s'|s),
\]
where $T_{\pi}(s'|s)$ and $T_{\beta}(s'|s)$ represent the transition probabilities from state $s$ to next-state $s'$ following policy $\pi$ and $\beta$ respectively, and for any state-next-state pair $(s,s')$, $\Delta T_{\pibehavior,\rho,\lambda}(s'|s)=\sum_{a\in A} T(s'|s,a)\pibehavior(a|s)\cdot \mix \cdot\frac{\rho(a|s)-\pibehavior(a|s)}{\pibehavior(a|s)}$.
\end{lemma}
\begin{proof}
Consider the following chain of equalities from matrix manipulations:
\[
\begin{split}
(I-\gamma T_{\pibehavior})^{-1}\cdot(I-\gamma T_{\pi})=&(I-\gamma T_{\pibehavior})^{-1}\cdot\left(I-\gamma  \left\{\sum_{a\in A} T(s'|s,a)\left(\pibehavior(a|s)+ \mix (\rho(a|s)-\pibehavior(a|s))\right)\right\}_{s,s'}\right)\\
=&\left(I-\gamma(I-\gamma T_{\pibehavior})^{-1}\left\{ \sum_{a\in A} T(s'|s,a)\mix (\rho(a|s)-\pibehavior(a|s))\right\}_{s,s'\in S}\right).
\end{split}
\]
By multiplying the matrix $(I-\gamma T_{\pi})^{-1}$ on both sides of the above expression, it implies that
\[
\begin{split}
(I-\gamma T_{\pibehavior})^{-1}-(I-\gamma T_{\pi})^{-1}=&-\left(\gamma(I-\gamma T_{\pibehavior})^{-1}\left\{ \sum_{a\in A} T(s'|s,a)\mix(\rho(a|s)-\pibehavior(a|s))\right\}_{s,s'\in S}(I-\gamma T_{\pi})^{-1}\right).
\end{split}
\]
Using the definition of $\Delta T_{\pibehavior,\rho,\lambda}$ completes the proof of this lemma.
\end{proof}

Using the above result, for any initial state $s\in S$, the value functions that are induced by $\pibehavior$ and $\pi$ have the following relationship:
\begin{equation}\label{eq:intermediate}
\begin{split}
V_{\pi}(s)-V_{\pibehavior}(s)=&\delta_{s_0=s}^\top(I-\gamma T_{\pi})^{-1} R_{\pi}-\delta_{s_0=s}^\top(I-\gamma T_{\pibehavior})^{-1} R_{\pibehavior}\\
=&\delta_{s_0=s}^\top\left((I-\gamma T_{\pi})^{-1} -(I-\gamma T_{\pibehavior})^{-1}\right) R_{\pi} + \delta_{s_0=s}^\top(I-\gamma T_{\pibehavior})^{-1}\left( R_\pi- R_{\pibehavior}\right)\\
= &\gamma\cdot\mathbb E_{T,\pibehavior}\left[\sum_{t=0}^\infty\gamma^t\mix[(s_t, a_t)]\cdot \frac{\rho(a_t|s_t)-\pibehavior(a_t|s_t)}{\pibehavior(a_t|s_t)}\cdot V_{\pi}(s_{t+1})\mid s_0=s\right] \\
&\quad+\mathbb E_{T,\pibehavior}\left[\sum_{t=0}^\infty\gamma^t \mix[(s_t, a_t)] \cdot \frac{\rho(a_t|s_t)-\pibehavior(a_t|s_t)}{\pibehavior(a_t|s_t)}\cdot R(s_t,a_t)\mid s_0=s\right]\\
= & \mathbb E_{T,\pibehavior}\left[\sum_{t=0}^\infty\gamma^t \mix[(s_t, a_t)] \cdot \frac{\rho(a_t|s_t)-\pibehavior(a_t|s_t)}{\pibehavior(a_t|s_t)}\cdot Q_\pi(s_t,a_t)\mid s_0=s\right]\\
=&\mathbb E_{T,\pibehavior}\left[\sum_{t=0}^\infty\gamma^t \mix[(s_t, a_t)] \cdot \frac{\rho(a_t|s_t)-\pibehavior(a_t|s_t)}{\pibehavior(a_t|s_t)}\cdot A_\pi(s_t,a_t)\mid s_0=s\right].
\end{split}
\end{equation}
The second equality follows from the fact that 
$
Q_\pi(s,a) = R(s,a) + \gamma\sum_{s'\in S}T(s'|s,a)V_\pi(s').
$
The third equality follows from the result in Lemma \ref{lemma:tech} and the fact that for any state $s\in S$, $(I-\gamma T_{\pi})^{-1}R_\pi(s)=V_\pi(s)$. The last equality is based on the fact of the confidence constraint that 
\[
\mathbb E_{T,\pibehavior}\left[\sum_{t=0}^\infty \gamma^t V_\pi(s_t)\cdot\sum_{a}\mix[(s_t, a)]\cdot(\rho(a|s_t)- \beta(a|s_t))\mid s_0=s\right] = 0,\,\,\forall s\in S.
\]

\subsection{Proof of Theorem \ref{thm:main}}
Denote by $V_\pi$ and $V_\pibehavior$ the vectors of value functions $V_\pi(s)$ and $V_\pibehavior(s)$ at every state $s\in S$ respectively. Re-writing the result in (\ref{eq:intermediate}) in matrix form, it can be expressed as
\[
V_{\pi}-V_{\pibehavior}=(I-\gamma T_{\pibehavior})^{-1}\left(\Delta R_{\pibehavior,\rho,\lambda}+\gamma\Delta T_{\pibehavior,\rho,\lambda}V_\pibehavior+\gamma\Delta T_{\pibehavior,\rho,\lambda}(V_\pi-V_\pibehavior)\right),
\]
where for any state $s\in S$, $\Delta R_{\pibehavior,\rho,\lambda}(s)=\sum_{a}\pibehavior(a|s)\cdot \mix \cdot \frac{\rho(a|s)-\pibehavior(a|s)}{\pibehavior(a|s)}\cdot R(s,a)$, and $\Delta T_{\pibehavior,\rho,\lambda}(s'|s)=\sum_{a\in A} T(s'|s,a)\pibehavior(a|s)\cdot \mix \cdot\frac{\rho(a|s)-\pibehavior(a|s)}{\pibehavior(a|s)}$. This expression implies that
\[
(I - (I-\gamma T_{\pibehavior})^{-1}\gamma\Delta T_{\pibehavior,\rho,\lambda})(V_{\pi}-V_{\pibehavior})=(I-\gamma T_{\pibehavior})^{-1}(\Delta R_{\pibehavior,\rho,\lambda}+\gamma\Delta T_{\pibehavior,\rho,\lambda}V_\pibehavior),
\]
which further implies that
\[
V_{\pi}-V_{\pibehavior}=(I - (I-\gamma T_{\pibehavior})^{-1}\gamma\Delta T_{\pibehavior,\rho,\lambda})^{-1}(I-\gamma T_{\pibehavior})^{-1}\left(\Delta R_{\pibehavior,\rho,\lambda}+\gamma\Delta T_{\pibehavior,\rho,\lambda}V_\pibehavior\right).
\]
Here based on the definition of $\Delta T_{\pibehavior,\rho,\lambda}$ and the confidence constraint, one can show that $(T_{\pibehavior}+\Delta T_{\pibehavior,\rho,\lambda})$ is a stochastic matrix (all the elements are non-negative, and $\sum_{s'\in S}(T_{\pibehavior}+\Delta T_{\pibehavior,\rho,\lambda})(s'|s)=1$, $\forall s\in S$). Therefore the matrix $(I - (I-\gamma T_{\pibehavior})^{-1}\gamma\Delta T_{\pibehavior,\rho,\lambda})$ is invertible. 

Using the matrix inversion lemma, one has the following equality:
\[
(I - (I-\gamma T_{\pibehavior})^{-1}\cdot\gamma\cdot\Delta T_{\pibehavior,\rho,\lambda})^{-1}=(I-\gamma (T_{\pibehavior}+\Delta T_{\pibehavior,\rho,\lambda}))^{-1}(I-\gamma T_{\pibehavior}).
\]
Therefore the difference of value function $V_{\pi}-V_{\pibehavior}$ can further be expressed as
\[
V_{\pi}-V_{\pibehavior}=(I-\gamma T_{\pibehavior}-\gamma\Delta T_{\pibehavior,\rho,\lambda})^{-1}\left(\Delta R_{\pibehavior,\rho,\lambda}+\gamma\Delta T_{\pibehavior,\rho,\lambda}V_\pibehavior\right).
\]
In other words, at any state $s\in S$, the corresponding value function $V_\pi(s)$ is given by the following expression:
\begin{equation}
\begin{split}
V_\pi(s)-V_\pibehavior(s)=&\mathbb E\left[\sum_{t=0}^\infty\gamma^t\left(\Delta R_{\pibehavior,\rho,\lambda}+\gamma\Delta T_{\pibehavior,\rho,\lambda}V_\pibehavior\right)(s_t)\mid T'_{\pibehavior,\rho},s_0=s\right]\\
=&\mathbb E\left[\sum_{t=0}^\infty\gamma^t\Delta Q_{\pibehavior,\rho,\lambda}(s_t,a_t)\mid T'_{\pibehavior,\rho},s_0=s\right]\\
=&\mathbb E\left[\sum_{t=0}^\infty\gamma^t\Delta A_{\pibehavior,\rho,\lambda}(s_t,a_t)\mid T'_{\pibehavior,\rho},s_0=s\right],
\end{split}
\end{equation}
where the transition probability $T'_{\pibehavior,\rho}(s'|s)=(T_{\pibehavior}+\Delta T_{\pibehavior,\rho,\lambda})(s'|s)$ is given by $\sum_{a\in A}\pibehavior(a|s)\cdot T(s'|s,a)\cdot\left(1+ \mix \cdot\frac{\rho(a|s)-\pibehavior(a|s)}{\pibehavior(a|s)}\right)$ at state-next-state pair $(s,s')$. By noticing that $T'_{\pibehavior,\rho}(s'|s)$ is indeed $T_\pi(s'|s)$ (the transition probability that is induced by residual policy $\pi$), the proof of Theorem \ref{thm:main} is completed.

    \newpage
    \section{Proofs for Results in Section~\ref{sec:bcrpo}}
\subsection{Proof of Theorem \ref{thm:cpi_2}}
Define the state-action discounted stationary distribution w.r.t. an arbitrary policy $\pi$ as 
$\visitpi(s,a)=(1-\gamma)\sum_{t=0}^\infty\gamma^t \mathbb P(s_t=s,a_t=a|s_0\sim  P_0,\pi)$ and its state-only counterpart as $\visitpi(s)=\sum_{a\in A}\visitpi(s,a)\pi(a|s)$.
Immediately one can write the difference of return (objective function of this problem) with the following chain of equalities/inequalities:
\begin{align*}
\E\left[\sum_{t=0}^\infty \gamma^t\Delta A_{\pibehavior,\rho,\lambda}(s_t)\mid T, \pi, s_0\sim P_0 \right]
&=\frac{1}{1-\gamma}\sum_{s\in S}\visitpi(s)\Delta A_{\pibehavior,\rho,\lambda}(s)\\
&=\frac{1}{1-\gamma}\sum_{s\in S}\visitbehavior(s)\Delta A_{\pibehavior,\rho,\lambda}(s)+ (\visitpi(s)-\visitbehavior(s))\Delta A_{\pibehavior,\rho,\lambda}(s).
\end{align*}

Recall that $\Delta T_{\pibehavior,\rho,\lambda}(s'|s)=\sum_{a\in A} T(s'|s,a)\pibehavior(a|s)\cdot \mix \cdot\frac{\rho(a|s)-\pibehavior(a|s)}{\pibehavior(a|s)}$. At any state $s\in S$, the difference of stationary distribution $\visitpi(s)-\visitbehavior(s)$ can be further expressed as
\[
\begin{split}
 (\visitpi-\visitbehavior&)(s)
 = P_0^\top\left((I-\gamma T_{\pi})^{-1}-(I-\gamma T_{\pibehavior})^{-1}\right)(s)\\
 =& P_0^\top\big((I-\gamma T_{\pibehavior})^{-1}+(I-\gamma T_{\pibehavior})^{-1}\gamma\Delta T_{\pibehavior,\rho,\lambda}(I-(I-\gamma T_{\pibehavior})^{-1}\gamma\Delta T_{\pibehavior,\rho,\lambda})^{-1}\cdot\\
&\qquad\qquad(I-\gamma T_{\pibehavior})^{-1}-(I-\gamma T_{\pibehavior})^{-1}\big)(s)\\
 =& P_0^\top\left((I-\gamma T_{\pibehavior})^{-1}\gamma\Delta T_{\pibehavior,\rho,\lambda}(I-(I-\gamma T_{\pibehavior})^{-1}\gamma\Delta T_{\pibehavior,\rho,\lambda})^{-1}(I-\gamma T_{\pibehavior})^{-1}\right)(s).
 \end{split}
\]
 
Let $D_{\pibehavior}=\left\{(1-\gamma)\E[\sum_{t=0}^\infty\gamma^t\mathbb P (s_t=s'|s_0=s,\pibehavior)]\right\}_{s,s'\in S}$ be the occupation measure matrix induced by $\pibehavior$. Combining the above arguments one has
\begin{align}
& \left|\langle\visitpi-\visitbehavior,\Delta A_{\pibehavior,\rho,\lambda}\rangle\right|\label{eq:abs_diff}\\
= &\left|\langle P_0^\top(I-\gamma T_{\pibehavior})^{-1}\gamma\Delta T_{\pibehavior,\rho,\lambda}(I-(I-\gamma T_{\pibehavior})^{-1}\gamma\Delta T_{\pibehavior,\rho,\lambda})^{-1}(I-\gamma T_{\pibehavior})^{-1},\Delta A_{\pibehavior,\rho,\lambda}\rangle\right|\nonumber\\
= &\frac{1}{1-\gamma}\left|\langle P_0^\top D_{\pibehavior}\gamma\Delta T_{\pibehavior,\rho,\lambda}((1-\gamma)I-D_\pibehavior\gamma\Delta T_{\pibehavior,\rho,\lambda})^{-1},D_{\pibehavior}\Delta A_{\pibehavior,\rho,\lambda}\rangle\right|\nonumber\\
= &\frac{\gamma}{1-\gamma}\left|\langle P_0^\top D_{\pibehavior}\Delta T_{\pibehavior,\rho,\lambda},\left( I -\gamma \left(I+D_\pibehavior\Delta T_{\pibehavior,\rho,\lambda}\right)\right)^{-1}D_{\pibehavior}\Delta A_{\pibehavior,\rho,\lambda}\rangle\right|.\nonumber
\end{align}
Now $I+D_\pibehavior\Delta T_{\pibehavior,\rho,\lambda}$ is a stochastic matrix, which is because for any state $s\in S$,
\begin{align*}
(D_\pibehavior\Delta T_{\pibehavior,\rho,\lambda}e)(s) &= (1-\gamma) \E \left[\sum_{t=0}^\infty \gamma^t\sum_{s'}\Delta T_{\pibehavior,\rho,\lambda}(s'|s)\mid T, \beta, s_0=s \right]\\
&=(1-\gamma)\E\left[\sum_{t=0}^\infty \gamma^t\sum_{a\in A} \mix \cdot(\rho(a|s)-\pibehavior(a|s))\mid T, \beta, s_0=s \right]\\
&=0.
\end{align*}
Using this property one can upper bound the magnitude of each element of matrix $\left( I -\gamma \left(I+D_\pibehavior\Delta T_{\pibehavior,\rho,\lambda}\right)\right)^{-1}$ by $\frac{1}{1-\gamma}$. Therefore, using Holder inequality one can further upper bound the expression in (\ref{eq:abs_diff}) as follows:
\begin{align*}
\left|\langle\visitpi-\visitbehavior,\Delta A_{\pibehavior,\rho,\lambda}\rangle\right|
&\leq \frac{\gamma}{(1-\gamma)^2}\left\| P_0^\top D_{\pibehavior}\Delta T_{\pibehavior,\rho,\lambda}\right\|_1\cdot\left\|D_{\pibehavior}\Delta A_{\pibehavior,\rho,\lambda}\right\|_\infty\\
&\leq \frac{\gamma}{(1-\gamma)^2}\sum_{s\in S,a\in A}\visitbehavior(s)\pibehavior(a|s)\cdot \mix \cdot\frac{|\rho(a|s)-\pibehavior(a|s)|}{\pibehavior(a|s)}\cdot \max_{s_0\in S}\sum_{s\in S}\visitbehavior(s|s_0)\left|\Delta A_{\pibehavior,\rho,\lambda}\right|.
 \end{align*}
Plugging in the definition of the discounted occupation measure w.r.t. $\pibehavior$ into the above expression, the proof is this theorem is completed. 
    
    \newpage
    \section{Practical Implementation Details of BRPO}
In this section, we discuss several practical techniques to further boost training stability and effectiveness of \BCRPO{}.

\paragraph{Improving CPI with Optimal Advantage}
The derivation in Section \ref{sec:bcrpo} and Section \ref{sec:algorithm} shows that optimizing  $\mathcal{L}_{\pibehavior,\rho,\lambda}$ finds a residual policy that performs no-worse than the behavior policy (modulo any Lagrangian approximation error).
While we argue that this optimization is less conservative than existing methods (like TRPO) due to the state-action-dependent learned confidence, it might not be aggressive enough in leveraging the function approximation to generalize to unseen state and action.
One major reason is that by design, $\mathcal{L}_{\pibehavior,\rho,\lambda}$ only uses the long-term value of $\pibehavior$ (in the form of $A_\pibehavior$), in order to circumvent the issue of bad generalization to unseen state-action pair. This also makes policy improvement local to $\pibehavior$. This is a fundamental challenge of batch RL, but can be relaxed depending on the domain.
As a remedy to this issue, by a convex ensemble of the results from Lemma \ref{lem:main} and Theorem \ref{thm:main} (with any combination weight $\mu\in[0,1]$), notice that the difference-return also satisfies
\[
J_\pi - J_\beta \ge \frac{1}{1-\gamma} \left( \widetilde{L}^{\prime}_{\mu,\pibehavior,\rho,\lambda} - \frac{\gamma(1-\mu)}{1-\gamma} L^{\prime\prime}_{\pibehavior,\rho,\lambda}\max_{s_0\in S}\, L^{\prime\prime\prime}_{\pibehavior,\rho,\lambda}(s_0)
    \right),
\]
where 
\[
\widetilde{L}^{\prime}_{\mu,\pibehavior,\rho,\lambda}:= \E_{(s,a)\sim \visitbehavior}\left[
    \mix \cdot\frac{\rho(a|s)-\pibehavior(a|s)}{\pibehavior(a|s)} \cdot
    W(s,a)
    \right]
    \]
    with a \emph{weighted advantage} function $W:= (1-\mu) \adv[\pibehavior] + \mu \adv[\pi]$.
Therefore, without loss of generality one can replace $A_\pibehavior$ in
$\mathcal{L}_{\pibehavior,\rho,\lambda}$ with the weighted advantage function. Furthermore, to avoid estimating $A_\pi$ at each policy update and assuming that CPI eventually finds $\pi\rightarrow\pi^*$, one may directly estimate the optimal weighted advantage function 
$$W_{\pi^*}(s,a)= Q_{\mu,\beta,\pi^*}(s,a)-V_{\mu,\beta,\pi^*}(s),$$ in which the value function $Q_{\mu,\beta,\pi^*}$ is a Bellman fixed-point of  $Q(s,a)=R(s,a)+\gamma\sum_{s'}T(s'|s,a) V(s')$, with
\[
V(s) = (1-\mu)\mathbb E_{a\sim\pibehavior}[Q(s,a)]+\mu \max_{a}Q(s,a).
\]
This approach of combining the optimal Bellman operator with the on-policy counterpart belongs to the general class of hybrid on/off-policy RL algorithms \citep{o2016combining}.

Therefore,
we learn an advantage function $W$ that is a weighted combination of $A_\pibehavior$ and $A_{\pi^*}$.
Using the batch data $B$, the expected advantage $A_\pibehavior$ can be learned with any critic-learning technique, such as SARSA~\citep{sutton2018reinforcement}.
We can learn $A_{\pi^*}$ by DQN~\citep{mnih2013playing} or other Q-learning algorithm.
We provide pseudo-code of our \BCRPO{} algorithm in Algorithm~\ref{alg:train}.
    
    \newpage
    \section{Experimental Details} \label{appendix:exp_details}
    This section describes more details about our experimental setup to evaluate the algorithms.

\subsection{Behavior policy}
We train the behavior policy using DQN, using architecture and hyper-parameters specified in Section~\ref{sec:hyper-parameter}. The behavior policy was trained for each task until the performance reaches around 75\% of the optimal performance similar to ~\cite{fujimoto2018off} and \cite{kumar2019stabilizing}. 
Specifically, we trained the behavior policy for 100,000 steps for \luna{}, and 50,000 steps for \cart{} and \acro{}. We used two-layers MLP with FC($32$)-FC($16$). The replay buffer size is $500,000$ and batch size is $64$.
The performance of the behavior policies are given in Table~\ref{table:exp1_best_mean_new}.


\subsection{Hyperparameters}\label{sec:hyper-parameter}
For fair comparison, we generally used the same set of hyper-parameters and
architecture across all methods and experiments, which are defined in Table~\ref{table:hyper_params} and Table~\ref{table:hyper_params_BRPO}. Similar to the behavior policy, we used two-layers MLP with FC(32)-FC(16) for all the critic agents and Behavioral cloning agent's policy.
The final hyperparameters are found using grid search, with candidate set specified in Table~\ref{table:hyper_params} and Table~\ref{table:hyper_params_BRPO}.

\begin{table}[ht]
\centering
\begin{tabular}{|l|c|c|}
\hline
\textbf{Hyperparameters for BC, BCQ, SARSA, DQN, and KL-Q} & \textbf{Sweep range} & \textbf{Final value}\\ [0.5ex]
\hline
\hline
Soft target update rate ($\tau$) & - & 0.5 \\
\hline
Soft target update period & 150, 500, 1500 & 500 \\
\hline
Discount factor & - & 0.99\\
\hline
Mini-batch size & - & 64 \\
\hline
Q-function learning rates & 0.0003, 0.001, 0.002 & 0.001\\
\hline
Neural network optimizer & - & Adam \\
\hline
[BCQ] Behavior policy threshold ($\tau$ in \cite{fujimoto2019benchmarking}) & 0.1, 0.2, 0.3, 0.4 & 0.3 \\
\hline
[SPIBB] Bootstrapping set threshold & 0.1, 0.2, 0.3, 0.4 & 0.2 \\
\hline
[KL-Q] KL-regularization weight & 0.01, 0.03, 0.1, 0.3 & 0.1 \\
\hline
\end{tabular}
\caption{The range of hyperparameters sweeped over and the final hyperparameters used for the baselines (BC, BCQ, SARSA, DQN, and KL-Q).}
\label{table:hyper_params}
\end{table}

\begin{table}[ht]
\centering
\begin{tabular}{|l|c|c|}
\hline
\textbf{Hyperparameters for BRPO and BRPO-C} & \textbf{Sweep range} & \textbf{Final value}\\ [0.5ex]
\hline
\hline
Soft target update rate ($\tau$) & - & 0.5 \\
\hline
Soft target update period & 150, 500, 1500 & 500 \\
\hline
Discount factor & - & 0.99\\
\hline
Mini-batch size & - & 64 \\
\hline
Q-function learning rates & 0.0003, 0.001, 0.002 & 0.001\\
\hline
Neural network optimizer & - & Adam \\
\hline
[BRPO, BRPO-C] Mixing $\mu$ (for critic training) & 0.1, 0.5, 0.9 & 0.9\\
\hline
[BRPO] Confidence $\lambda_{\phi}(s, a)$ learning rates & 0.0002, 0.0001 & 0.0001 \\
\hline
[BRPO-C] Constant $\lambda$ (for constant residual policy) & 0.25, 0.33, 0.5, 0.66, 0.75 & 0.5\\
\hline
\end{tabular}
\caption{The range of hyperparameters sweeped over and the final hyperparameters used for the proposed methods (BRPO and BRPO-C).}
\label{table:hyper_params_BRPO}
\end{table}

    \section{Additional Results} \label{appendix:exp_results}
    Here are the learning curves for each environment and behavior policy over the course of batch training.

\label{additional_plots}
\begin{figure*}[!ht]
    \vspace*{-3pt}
    \centering
    \includegraphics[draft=false,width=0.99\linewidth]{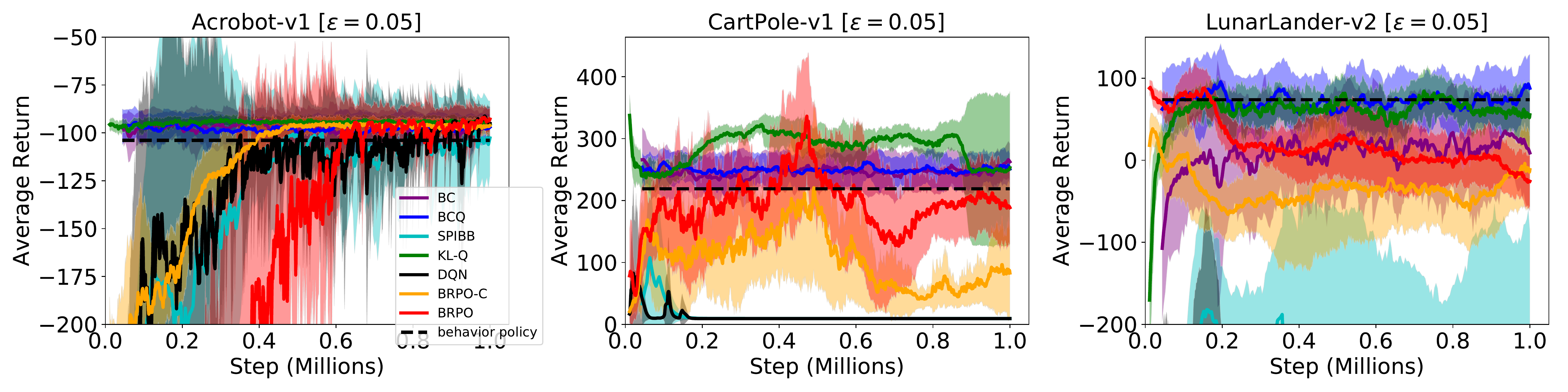}\\
    \includegraphics[draft=false,width=0.99\linewidth]{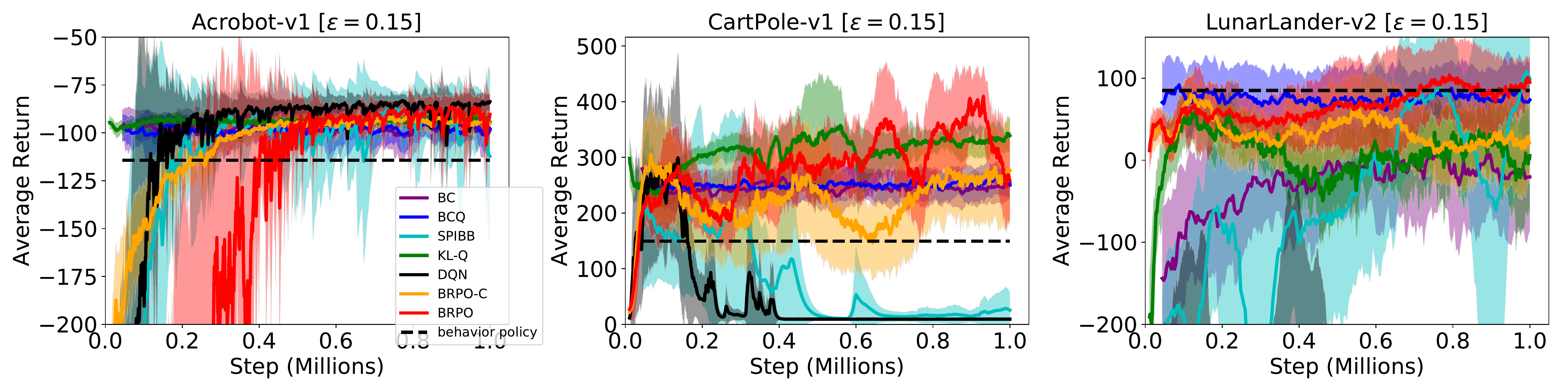}\\
    \includegraphics[draft=false,width=0.99\linewidth]{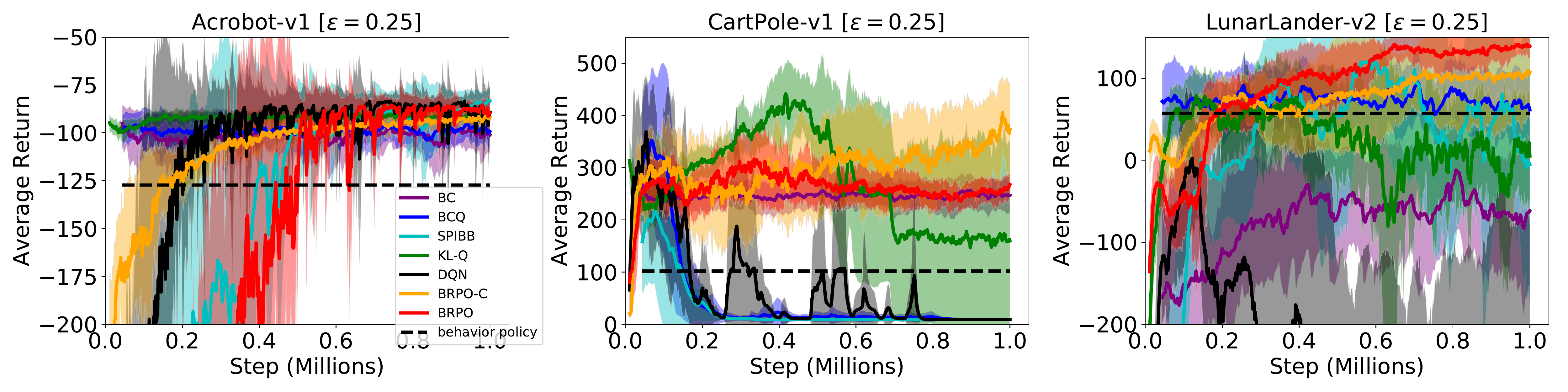}\\
    \includegraphics[draft=false,width=0.99\linewidth]{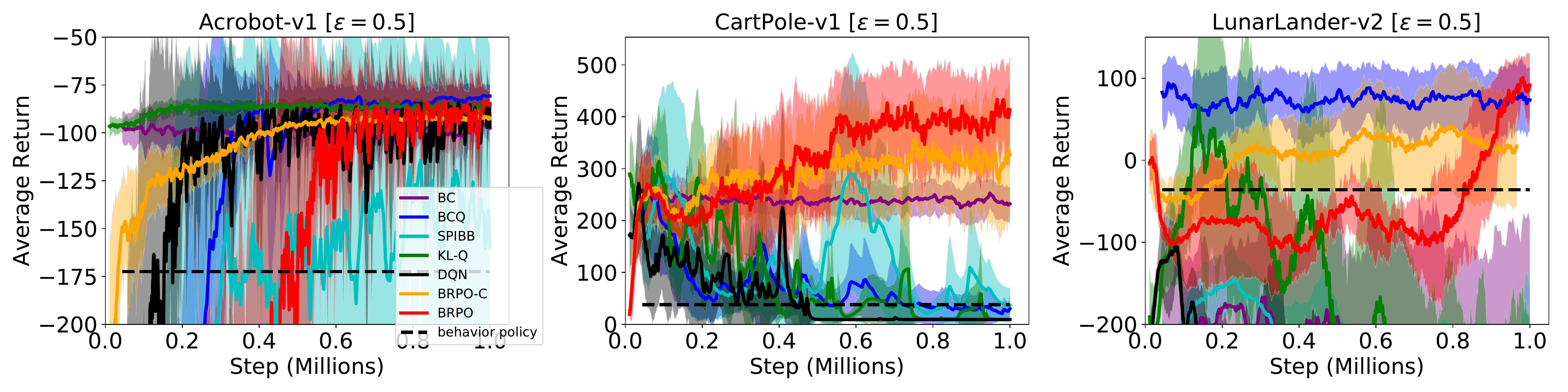}\\
    \includegraphics[draft=false,width=0.99\linewidth]{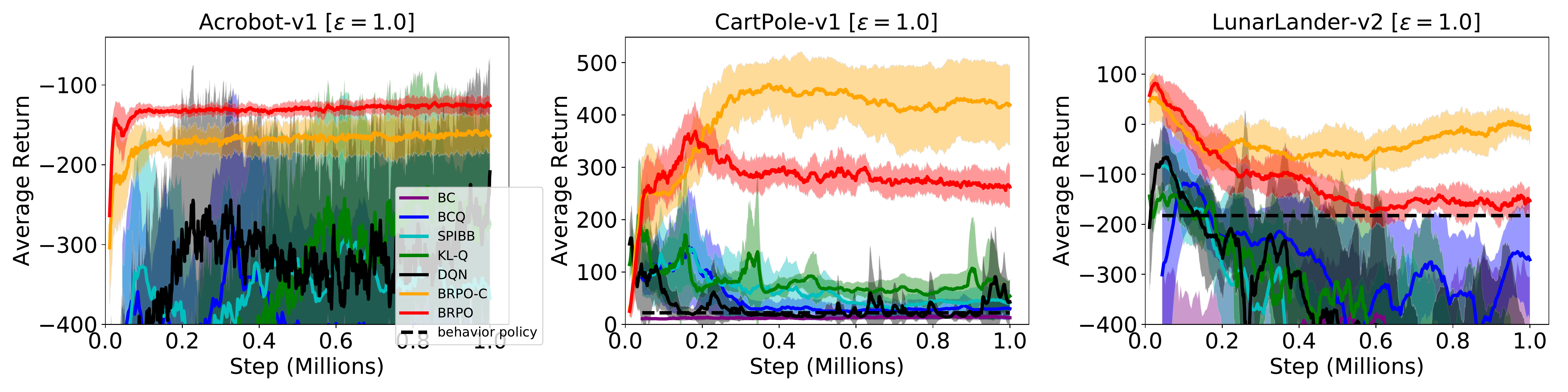}
    \vspace{-8pt}
    \caption{
        The mean $\pm$ standard error (shadowed area) of average return for each environment and behavior policy $\varepsilon$-exploration.
    }
    \label{fig:exp1_best_mean}
    \vspace{-7pt}
\end{figure*}

}
\fi

\end{document}